\documentclass[sigconf]{acmart}

\AtBeginDocument{%
  \providecommand\BibTeX{{%
    \normalfont B\kern-0.5em{\scshape i\kern-0.25em b}\kern-0.8em\TeX}}}

\usepackage{courier}
\usepackage{mathtools}
\usepackage{graphicx}
\usepackage{float}
\usepackage{subfigure}
\usepackage[ruled,linesnumbered]{algorithm2e}
\usepackage{amsmath}  
\usepackage{booktabs}
\usepackage{multirow}
\usepackage{color}
\usepackage{amsfonts}
\usepackage{caption}
\usepackage{subfigure}
\usepackage{enumerate}
\usepackage{enumitem}
\usepackage{bm}
\usepackage{stfloats}
\newtheorem{definition}{Definition}
\newtheorem{theorem}{Theorem}
\copyrightyear{2020}
\acmYear{2020}
\setcopyright{iw3c2w3}
\acmConference[WWW '20]{Proceedings of The Web Conference 2020}{April 20--24, 2020}{Taipei, Taiwan}
\acmBooktitle{Proceedings of The Web Conference 2020 (WWW '20), April 20--24, 2020, Taipei, Taiwan}
\acmPrice{}
\acmDOI{10.1145/3366423.3380214}
\acmISBN{978-1-4503-7023-3/20/04}

\begin{document}

%% The "title" command has an optional parameter,
%% allowing the author to define a "short title" to be used in page headers.
\title{Structural Deep Clustering Network}

%% The "author" command and its associated commands are used to define
%% the authors and their affiliations.
%% Of note is the shared affiliation of the first two authors, and the
%% "authornote" and "authornotemark" commands
%% used to denote shared contribution to the research.

 \author{Deyu Bo}
 \authornote{Both authors contributed equally to this research.}
 \affiliation{
   \institution{Beijing University of Posts and \\ Telecommunications}
   \city{Beijing}
   \country{China}
   }
 \email{bodeyu@bupt.edu.cn}

 \author{Xiao Wang}
 \authornotemark[1]
 \affiliation{
   \institution{Beijing University of Posts and \\ Telecommunications}
   \city{Beijing}
   \country{China}
   }
 \email{xiaowang@bupt.edu.cn}
 
 \author{Chuan Shi}
 \authornote{Corresponding author}
 \affiliation{
   \institution{Beijing University of Posts and \\ Telecommunications}
   \city{Beijing}
   \country{China}
   }
 \email{shichuan@bupt.edu.cn}

 \author{Meiqi Zhu}
 \affiliation{
   \institution{Beijing University of Posts and \\ Telecommunications}
   \city{Beijing}
   \country{China}
   }
 \email{zhumeiqi@bupt.edu.cn}

 \author{Emiao Lu}
 \affiliation{
   \institution{Tencent}
   \city{Shenzhen}
   \country{China}
   }
 \email{emiao.lu@gmail.com}
 
 \author{Peng Cui}
 \affiliation{
   \institution{Tsinghua University}
   \city{Beijing}
   \country{China}
   }
 \email{cuip@tsinghua.edu.cn}

%% By default, the full list of authors will be used in the page
%% headers. Often, this list is too long, and will overlap
%% other information printed in the page headers. This command allows
%% the author to define a more concise list
%% of authors' names for this purpose.
% \renewcommand{\shortauthors}{Trovato and Tobin, et al.}

%%
%% The abstract is a short summary of the work to be presented in the
%% article.
\begin{abstract}
Clustering is a fundamental task in data analysis. Recently, deep clustering, which derives inspiration primarily from deep learning approaches, achieves state-of-the-art performance and has attracted considerable attention. Current deep clustering methods usually boost the clustering results by means of the powerful representation ability of deep learning, e.g., autoencoder, suggesting that learning an effective representation for clustering is a crucial requirement. The strength of deep clustering methods is to extract the useful representations from the data itself, rather than the structure of data, which receives scarce attention in representation learning. Motivated by the great success of Graph Convolutional Network (GCN) in encoding the graph structure, we propose a Structural Deep Clustering Network (SDCN) to integrate the structural information into deep clustering. Specifically, we design a delivery operator to transfer the representations learned by autoencoder to the corresponding GCN layer, and a dual self-supervised mechanism to unify these two different deep neural architectures and guide the update of the whole model. In this way, the multiple structures of data, from low-order to high-order, are naturally combined with the multiple representations learned by autoencoder. Furthermore, we theoretically analyze the delivery operator, i.e., with the delivery operator, GCN improves the autoencoder-specific representation as a high-order graph regularization constraint and autoencoder helps alleviate the over-smoothing problem in GCN. Through comprehensive experiments, we demonstrate that our propose model can consistently perform better over the state-of-the-art techniques.
\end{abstract}

%%
%% The code below is generated by the tool at http://dl.acm.org/ccs.cfm.
%% Please copy and paste the code instead of the example below.
%%
%\begin{CCSXML}
%<ccs2012>
% <concept>
%  <concept_id>10010520.10010553.10010562</concept_id>
%  <concept_desc>Computer systems organization~Embedded systems</concept_desc>
%  <concept_significance>500</concept_significance>
% </concept>
% <concept>
%  <concept_id>10010520.10010575.10010755</concept_id>
%  <concept_desc>Computer systems organization~Redundancy</concept_desc>
%  <concept_significance>300</concept_significance>
% </concept>
% <concept>
%  <concept_id>10010520.10010553.10010554</concept_id>
%  <concept_desc>Computer systems organization~Robotics</concept_desc>
%  <concept_significance>100</concept_significance>
% </concept>
% <concept>
%  <concept_id>10003033.10003083.10003095</concept_id>
%  <concept_desc>Networks~Network reliability</concept_desc>
%  <concept_significance>100</concept_significance>
% </concept>
%</ccs2012>
%\end{CCSXML}
%
%\ccsdesc[500]{Computer systems organization~Embedded systems}
%\ccsdesc[300]{Computer systems organization~Redundancy}
%\ccsdesc{Computer systems organization~Robotics}
%\ccsdesc[100]{Networks~Network reliability}

%%
%% Keywords. The author(s) should pick words that accurately describe
%% the work being presented. Separate the keywords with commas.

%% This command processes the author and affiliation and title
%% information and builds the first part of the formatted document.
\maketitle

\section{Introduction}
Clustering, one of the most fundamental data analysis tasks, is to group similar samples into the same category \cite{hartigan1979algorithm,ng2002spectral}. Over the past decades, a large family of clustering algorithms has been developed and successfully applied to various real-world applications, such as image clustering \cite{yang2010image} and text clustering \cite{aggarwal2012survey}. Recently, the breakthroughs in deep learning have led to a paradigm shift in artificial intelligence and machine learning, achieving great success on many important tasks, including clustering. Therefore, the deep clustering has caught significant attention \cite{hatcher2018survey}. The basic idea of deep clustering is to integrate the objective of clustering into the powerful representation ability of deep learning. Hence learning an effective data representation is a crucial prerequisite for deep clustering. For example, \cite{yang2017towards} uses the representation learned by autoencoder in $K$-means; \cite{xie2016unsupervised, guo2017improved} leverage a clustering loss to help autoencoder learn the data representation with high cluster cohesion \cite{rousseeuw1987silhouettes}, and \cite{jiang2016variational} uses a variational autoencoder to learn better data representation for clustering. To date, deep clustering methods have achieved state-of-the-art performance and become the de facto clustering methods.

Despite the success of deep clustering, they usually focus on the characteristic of data itself, and thus seldom take the structure of data into account when learning the representation. Notably, the importance of considering the relationship among data samples has been well recognized by previous literatures and results in data representation field. Such structure reveals the latent similarity among samples, and therefore provides a valuable guide on learning the representation. One typical method is the spectral clustering \cite{ng2002spectral}, which treats the samples as the nodes in weighted graph and uses graph structure of data for clustering. Recently, the emerging Graph Convolutional Networks (GCN) \cite{kipf2016semi} also encode both of the graph structure and node attributes for node representation. In summary, the structural information plays a crucial role in data representation learning. However, it has seldom been applied for deep clustering.

In reality, integrating structural information into deep clustering usually needs to address the following two problems.
(1) \textit{What structural information should be considered in deep clustering?} It is well known that the structural information indicates the underlying similarity among data samples. However, the structure of data is usually very complex, i.e., there is not only the direct relationship between samples (also known as first-order structure), but also the high-order structure. The high-order structure imposes the similarity constraint from more than one-hop relationship between samples. Taking the second-order structure as an example, it implies that for two samples with no direct relationship, if they have many common neighbor samples, they should still have similar representations. When the structure of data is sparse, which always holds in practice, the high-order structure is of particular importance. Therefore, only utilizing the low-order structure in deep clustering is far from sufficient, and how to effectively consider higher-order structure is the first problem;
(2) \textit{What is the relation between the structural information and deep clustering?} The basic component of deep clustering is the Deep Neural Network (DNN), e.g. autoencoder. The network architecture of autoencoder is very complex, consisting of multiple layers. Each layer captures different latent information. And there are also various types of structural information between data. Therefore, what is the relation between different structures and different layers in autoencoder? One can use the structure to regularize the representation learned by the autoencoder in some way, however, on the other hand, one can also directly learn the representation from the structure itself. How to elegantly combine the structure of data with the autoencoder structure is another problem.

In order to capture the structural information, we first construct a $K$-Nearest Neighbor (KNN) graph, which is able to reveal the underlying structure of the data\cite{LiuXTZ19, LiuT17}. To capture the low-order and high-order structural information from the KNN graph, we propose a GCN module, consisting of multiple graph convolutional layers, to learn the GCN-specific representation.

% \footnotetext[1]{If there exists an original graph between the raw data, e.g. citation network, it can be used directly.}

In order to introduce structural information into deep clustering, we introduce an autoencoder module to learn the autoencoder-specific representation from the raw data, and propose a delivery operator to combine it with the GCN-specific representation. We theoretically prove that the delivery operator is able to assist the integration between autoencoder and GCN better. In particular, we prove that GCN provides an approximate second-order graph regularization for the representation learned by autoencoder, and the representation learned by autoencoder can alleviate the over-smoothing issue in GCN.

Finally, because both of the autoencoder and GCN modules will output the representations, we propose a dual self-supervised module to uniformly guide these two modules. Through the dual self-supervised module, the whole model can be trained in an end-to-end manner for clustering task. 

In summary, we highlight the main contributions as follows:

\begin{itemize}%[leftmargin=*]
\item We propose a novel Structural Deep Clustering Network (SDCN) for deep clustering. The proposed SDCN effectively combines the strengths of both autoencoder and GCN with a novel delivery operator and a dual self-supervised module. To the best of our knowledge, this is the first time to apply structural information into deep clustering explicitly.

\item We give a theoretical analysis of our proposed SDCN and prove that GCN provides an approximate second-order graph regularization for the DNN representations and 
the data representation learned in SDCN is equivalent to the sum of the representations with different-order structural information. Based on our theoretical analysis, the over-smoothing issue of GCN module in SDCN will be effectively alleviated.

\item Extensive experiments on six real-world datasets demonstrate the superiority of SDCN in comparison with the state-of-the-art techniques. Specifically, SDCN achieves significant improvements (17\% on NMI, 28\% on ARI) over the best baseline method.
\end{itemize}

\section{RELATED WORK}
In this section, we introduce the most related work: deep clustering and graph clustering with GCN.

Deep clustering methods aim to combine the deep representation learning with the clustering objective. For example, 
\cite{yang2017towards} proposes deep clustering network, using the loss function of $K$-means to help autoencoder learn a "K-means-friendly" data representation. 
Deep embedding clustering \cite{xie2016unsupervised} designs a KL-divergence loss to make the representation learned by autoencoder surround the cluster centers closer, thus improving the cluster cohesion.
Improved deep embedding clustering \cite{guo2017improved} adds a reconstruction loss to the objective of DEC as a constraint to help autoencoder learn a better data representation.
Variational deep embedding \cite{jiang2016variational} is able to model the data generation process and clusters jointly by using a deep variational autoencoder, so as to achieve better clustering results.
\cite{ji2017deep} proposes deep subspace clustering networks, which uses a novel self-expressive layer between the encoder and the decoder. It is able to mimic the "self-expressiveness" property in subspace clustering, thus obtaining a more expressive representation.
DeepCluster \cite{caron2018deep} treats the clustering results as pseudo labels so that it can be applied in training deep neural network with large datasets.
However, all of these methods only focus on learning the representation of data from the samples themselves. Another important information in learning representation, the structure of data, is largely ignored by these methods.

To cope with the structural information underlying the data, some GCN-based clustering methods have been widely applied. For instance,
\cite{kipf2016variational} proposes graph autoencoder and graph variation autoencoder, which uses GCN as an encoder to integrate graph structure into node features to learn the nodes embedding.
Deep attentional embedded graph clustering \cite{wang2019attributed} uses an attention network to capture the importance of the neighboring nodes and employs the KL-divergence loss from DEC to supervise the training process of graph clustering. 
All GCN-based clustering methods mentioned above rely on reconstructing the adjacency matrix to update the model, and those methods can only learn data representations from the graph structure, which ignores the characteristic of the data itself. However, the performance of this type of methods might be limited to the overlapping between community structure.

\section{THE PROPOSED MODEL}
In this section, we introduce our proposed structural deep clustering network, where the overall framework is shown in Figure 1. We first construct a KNN graph based on the raw data. Then we input the raw data and KNN graph into autoencoder and GCN, respectively. We connect each layer of autoencoder with the corresponding layer of GCN, so that we can integrate the autoencoder-specific representation into structure-aware representation by a delivery operator. Meanwhile, we propose a dual self-supervised mechanism to supervise the training progress of autoencoder and GCN. We will describe our proposed model in detail in the following.

\begin{figure*}
\centering
\includegraphics[width=0.75\textwidth]{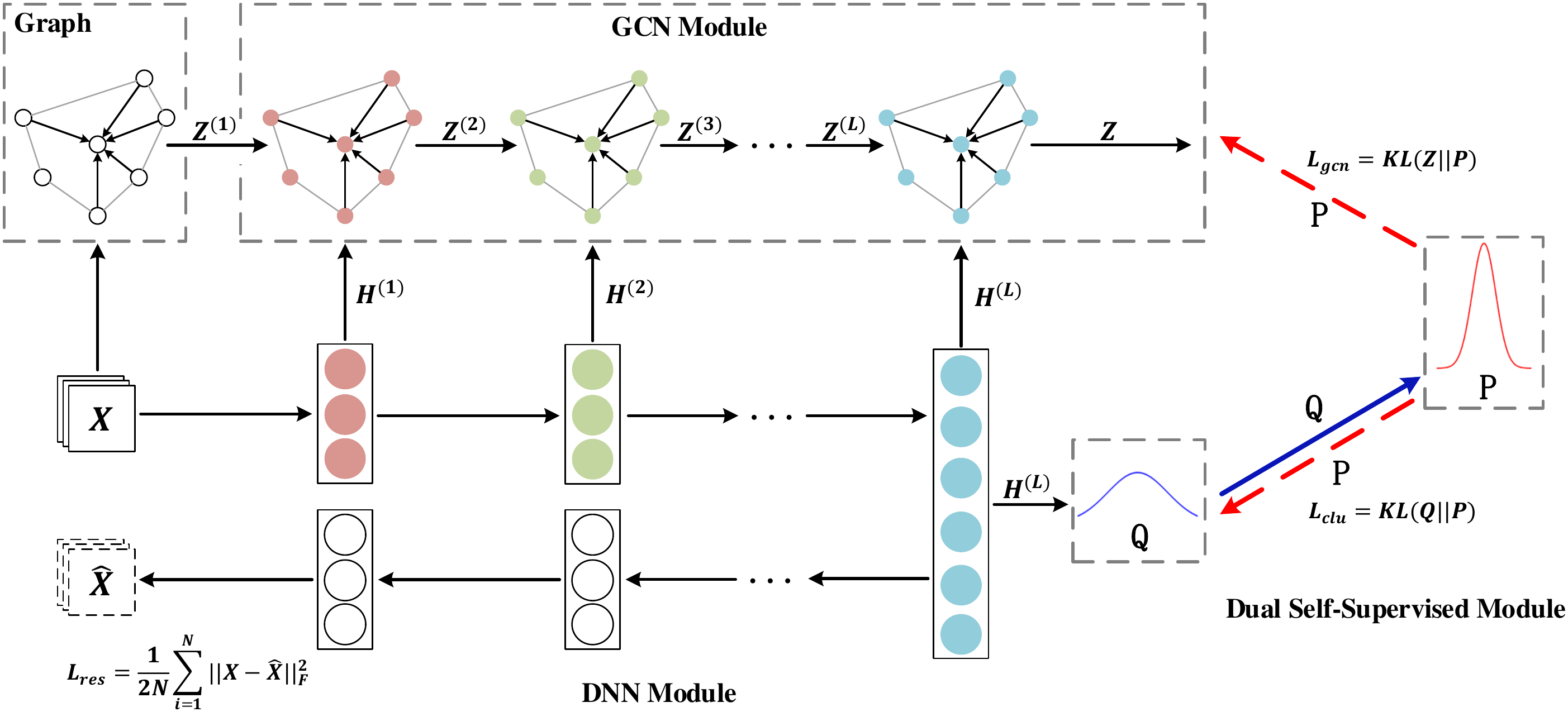}
\caption{The framework of our proposed SDCN. $\mathbf{X}$, $\mathbf{\hat{X}}$ are the input data and the reconstructed data, respectively. $\mathbf{H}^{(\ell)}$ and $\mathbf{Z}^{(\ell)}$ are the representations in the $\ell$-th layer in the DNN and GCN module, respectively. Different colors represent different representations $\mathbf{H}^{(\ell)}$, learned the by DNN module. The blue solid line represents that target distribution $\mathbf{P}$ is calculated by the distribution $\mathbf{Q}$ and the two red dotted lines represent the dual self-supervised mechanism. The target distribution $\mathbf{P}$ to guide the update of the DNN module and the GCN module at the same time.}
\label{model}
\end{figure*}

\subsection{KNN Graph}
Assume that we have the raw data $\mathbf{X}\in \mathbb{R}^{N\times d}$, where each row $\mathbf{x}_i$ represents the $i$-th sample, and $N$ is
the number of samples and $\textit{d}$ is the dimension. For each sample, we first find its top-$K$ similar neighbors and set edges to connect it with its neighbors.
There are many ways to calculate the similarity matrix $\mathbf{S}\in\mathbb{R}^{N\times N}$ of the samples. Here we list two popular approaches we used in constructing the KNN graph:

\begin{enumerate}%[leftmargin=*]
\item[1)] \textbf{Heat Kernel}. The similarity between samples $\textit{i}$ and $\textit{j}$ is calculated by:
\begin{align}
\mathbf{S}_{ij}=e^{-\frac{\left \| \mathbf{x}_{i}-\mathbf{x}_{j} \right \|^{2}}{t}},
\end{align}
where \textit{t} is the time parameter in heat conduction equation. For continuous data, e.g., images.
\item[2)] \textbf{Dot-product}.The similarity between samples $\textit{i}$ and $\textit{j}$ is calculated by:
\begin{align}
\mathbf{S}_{ij}=\mathbf{x}^{T}_{j}\mathbf{x}_{i}.
\end{align}
For discrete data, e.g., bag-of-words, we use the dot-product similarity so that the similarity is related to the number of identical words only.
\end{enumerate}
After calculating the similarity matrix $\mathbf{S}$, we select the top-$K$ similarity points of each sample as its neighbors to construct an undirected $K$-nearest neighbor graph. In this way, we can get the adjacency matrix $\mathbf{A}$ from the non-graph data.

\subsection{DNN Module}
As we mentioned before, learning an effective data representation is of great importance to deep clustering.
There are several alternative unsupervised methods for different types of data to learn representations. For example, denoising autoencoder \cite{vincent2008extracting}, convolutional autoencoder \cite{masci2011stacked}, LSTM encoder-decoder \cite{malhotra2016lstm} and adversarial autoencoder \cite{makhzani2015adversarial}. They are variations of the basic autoencoder \cite{hinton2006reducing}. In this paper, for the sake of generality, we employ the basic autoencoder to learn the representations of the raw data in order to accommodate for different kinds of data characteristics. We assume that there are $L$ layers in the autoencoder and $\ell$ represents the layer number. Specifically, the representation learned by the $\ell$-th layer in encoder part, $\mathbf{H}^{(\ell)}$, can be obtained as follows:
\begin{align}
\mathbf{H}^{(\ell)}=\phi\left(\mathbf{W}^{(\ell)}_{e} \mathbf{H}^{(\ell-1)} + \mathbf{b}^{(\ell)}_{e}\right),
\end{align}
where $\phi$ is the activation function of the fully connected layers such as Relu \cite{nair2010rectified} or Sigmoid function,
$\mathbf{W}^{(\ell)}_{e}$ and $\mathbf{b}^{(\ell)}_{e}$ are the weight matrix and bias of the $\ell$-th layer in the encoder, respectively. Besides, we denote $\mathbf{H}^{(0)}$ as the raw data $\mathbf{X}$.

The encoder part is followed by the decoder part, which is to reconstruct the input data through several fully connected layers by the equation
\begin{align}
\mathbf{H}^{(\ell)}=\phi\left(\mathbf{W}^{(\ell)}_{d} \mathbf{H}^{(\ell-1)} + \mathbf{b}^{(\ell)}_{d}\right),
\end{align}
where $\mathbf{W}^{(\ell)}_{d}$ and $\mathbf{b}^{(\ell)}_{d}$ are the weight matrix and bias of the $\ell$-th layer in the decoder, respectively.

The output of the decoder part is the reconstruction of the raw data $\mathbf{\hat{X}}=\mathbf{H}^{(L)}$, which results in the following objective function:
\begin{align}
\mathcal{L}_{res}=\frac{1}{2N}\sum^{N}_{i=1} \left\| \mathbf{x}_{i}-\mathbf{\hat{x}}_{i} \right\|_{2}^{2}=\frac{1}{2N}||\mathbf{X}-\mathbf{\hat{X}}||_{F}^{2}.
\end{align}

\subsection{GCN Module}
Autoencoder is able to learn the useful representations from the data itself, e.g. $\mathbf{H}^{(1)}, \mathbf{H}^{(2)},\cdots,\mathbf{H}^{(L)}$, while ignoring the relationship between samples. In the section, we will introduce how to use the GCN module to propagate these representations generated by the DNN module. Once all the representations learned by DNN module are integrated into GCN, then the GCN-learnable representation will be able to accommodate for two different kinds of information, i.e., data itself and relationship between data. In particular, with the weight matrix $\mathbf{W}$, the representation learned by the $\ell$-th layer of GCN, $\mathbf{Z}^{(\ell)}$, can be obtained by the following convolutional operation:
\begin{align}
\mathbf{Z}^{(\ell)}=\phi(\mathbf{\widetilde{D}}^{-\frac{1}{2}} \mathbf{\widetilde{A}} \mathbf{\widetilde{D}}^{-\frac{1}{2}} \mathbf{Z}^{(\ell-1)} \mathbf{W}^{(\ell-1)}),
\label{gcn}
\end{align}
where $\mathbf{\widetilde{A}}=\mathbf{A}+\mathbf{I}$ and $\mathbf{\widetilde{D}}_{ii}=\sum_{j}\mathbf{\widetilde{A}_{ij}}$. $\mathbf{I}$ is the identity diagonal matrix of the adjacent matrix $\mathbf{A}$ for the self-loop in each node. As can be seen from Eq. \ref{gcn}, the representation $\mathbf{Z}^{(\ell-1)}$ will propagate through the normalized adjacency matrix $\mathbf{\widetilde{D}}^{-\frac{1}{2}} \mathbf{\widetilde{A}} \mathbf{\widetilde{D}}^{-\frac{1}{2}}$ to obtain the new representation $\mathbf{Z}^{(\ell)}$. Considering that the representation learned by autoencoder $\mathbf{H}^{(\ell-1)}$ is able to reconstruct the data itself and contain different valuable information, we combine the two representations $\mathbf{Z}^{(\ell-1)}$
and $\mathbf{H}^{(\ell-1)}$ together to get a more complete and powerful representation as follows:
\begin{align}
\mathbf{\widetilde{Z}}^{(\ell-1)}=(1-\epsilon) \mathbf{Z}^{(\ell-1)}+\epsilon \mathbf{H}^{(\ell-1)},
\label{add}
\end{align}
where $\epsilon$ is a balance coefficient, and we uniformly set it to 0.5 here.
In this way, we connect the autoencoder and GCN layer by layer.

Then we use $\mathbf{\widetilde{Z}}^{(\ell-1)}$ as the input of the $l$-th layer in GCN to generate the representation $\mathbf{Z}^{(\ell)}$:
\begin{align}
\mathbf{Z}^{(\ell)}=\phi\left(\mathbf{\widetilde{D}}^{-\frac{1}{2}} \mathbf{\widetilde{A}} \mathbf{\widetilde{D}}^{-\frac{1}{2}} \mathbf{\widetilde{Z}}^{(\ell-1)} \mathbf{W}^{(\ell-1)} \right).
\label{ll}
\end{align}
As we can see in Eq. \ref{ll}, the autoencoder-specific representation $\mathbf{H}^{(\ell-1)}$ will be propagated through the normailized adjacency matrix $\mathbf{\widetilde{D}}^{-\frac{1}{2}} \mathbf{\widetilde{A}} \mathbf{\widetilde{D}}^{-\frac{1}{2}}$. Because the representations learned by each DNN layer are different, to preserve information as much as possible, we transfer the representations learned from each DNN layer into a corresponding GCN layer for information propagation, as in Figure \ref{model}. The delivery operator works $L$ times in the whole model. We will theoretically analyze the advantages of this delivery operator in Section 3.5.

Note that, the input of the first layer GCN is the raw data $\mathbf{X}$:
\begin{align}
\mathbf{Z}^{(1)}=\phi(\mathbf{\widetilde{D}}^{-\frac{1}{2}} \mathbf{\widetilde{A}} \mathbf{\widetilde{D}}^{-\frac{1}{2}} \mathbf{X} \mathbf{W}^{(1)}).
\label{gcnfirst}
\end{align}

The last layer of the GCN module is a multiple classification layer with a softmax function:
\begin{align}
Z=softmax\left( \mathbf{\widetilde{D}}^{-\frac{1}{2}} \mathbf{\widetilde{A}} \mathbf{\widetilde{D}}^{-\frac{1}{2}} \mathbf{Z}^{(L)} \mathbf{W}^{(L)} \right).
\label{zij}
\end{align}
The result $z_{ij} \in Z$ indicates the probability sample $i$ belongs to cluster center $j$, and we can treat $Z$ as a probability distribution.

\subsection{Dual Self-Supervised Module}
Now, we have connected the autoencoder with GCN in the neural network architecture. However, they are not designed for the deep clustering. Basically, autoencoder is mainly used for data representation learning, which is an unsupervised learning scenario, while the traditional GCN is in the semi-supervised learning scenario. Both of them cannot be directly applied to the clustering problem. Here, we propose a dual self-supervised module, which unifies the autoencoder and GCN modules in a uniform framework and effectively trains the two modules end-to-end for clustering.

In particular, for the $i$-th sample and $j$-th cluster, we use the Student's t-distribution \cite{maaten2008visualizing} as a kernel to measure the similarity between the data representation $\mathbf{h}_{i}$ and the cluster center vector $\bm{\mu}_{j}$ as follows:
\begin{align}
q_{ij} = \frac{(1+\left\|\mathbf{h}_{i} - \bm{\mu}_{j}\right\|^{2}/v)^{-\frac{v+1}{2}}}{\sum_{j'}(1+\left\|\mathbf{h}_{i} - \bm{\mu}_{j'}\right\|^{2}/v)^{-\frac{v+1}{2}}},
\label{qij}
\end{align}
where $\mathbf{h}_{i}$ is the $i$-th row of $\mathbf{H}^{(L)}$, $\bm{\mu}_{j}$ is initialized by $K$-means on representations learned by pre-train autoencoder and $v$ are the degrees of freedom of the Student’s t-distribution. $q_{ij}$ can be considered as the probability of assigning sample $i$ to cluster $j$, i.e., a soft assignment. We treat $Q=[q_{ij}]$ as the distribution of the assignments of all samples and let $\alpha$=1 for all experiments.

After obtaining the clustering result distribution $Q$, we aim to optimize the data representation by learning from the high confidence assignments. Specifically, we want to make data representation closer to cluster centers, thus improving the cluster cohesion. Hence, we calculate a target distribution $P$ as follows:
\begin{align}
p_{ij}=\frac{q^{2}_{ij}/f_{j}}{\sum_{j'}q^{2}_{ij'}/f_{j'}},
\label{pij}
\end{align}
where $f_{j}=\sum_{i}q_{ij}$ are soft cluster frequencies. In the target distribution $P$, each assignment in $Q$ is squared and normalized so that the assignments will have higher confidence, leading to the following objective function:
\begin{align}
\mathcal{L}_{clu}=KL(P||Q)=\sum_{i}\sum_{j}p_{ij}log\frac{p_{ij}}{q_{ij}}.
\label{clu}
\end{align}
By minimizing the KL divergence loss between $Q$ and $P$ distributions, the target distribution $P$ can help the DNN module learn a better representation for clustering task, i.e., making the data representation surround the cluster centers closer. This is regarded as a self-supervised mechanism\footnotemark[1], because the target distribution $P$ is calculated by the distribution $Q$, and the $P$ distribution supervises the updating of the distribution $Q$ in turn.

\footnotetext[1]{Although some previous work tend to call this mechanism self-training, we prefer to use the term "self-supervised" to be consistent with the GCN training method.}

As for training the GCN module, one possible way is to treat the clustering assignments as the truth labels \cite{caron2018deep}. However, this strategy will bring noise and trivial solutions, and lead to the collapse of the whole model. As mentioned before, the GCN module will also provide a clustering assignment distribution $Z$. Therefore, we can use distribution $P$ to supervise distribution $Z$ as follows:
\begin{align}
\mathcal{L}_{gcn}=KL(P||Z)=\sum_{i}\sum_{j}p_{ij}log\frac{p_{ij}}{z_{ij}}.
\label{dual}
\end{align}
There are two advantages of the objective function: (1) compared with the traditional multi-classification loss function, KL divergence updates the entire model in a more "gentle" way to prevent the data representations from severe disturbances; (2) both GCN and DNN modules are unified in the same optimization target, making their results tend to be consistent in the training process.
Because the goal of the DNN module and GCN module is to approximate the target distribution $P$, which has a strong connection between the two modules, we call it a dual self-supervised mechanism.

\begin{algorithm}
    \caption{Training process of SDCN}
    \label{alg:SDCN}
    \KwIn{Input data: $X$, Graph: $\mathcal{G}$, Number of clusters: $K$, Maximum iterations: $MaxIter$\;}
    \KwOut{Clustering results $\mathbf{R}$\;}
    Initialize $\mathbf{W}^{(\ell)}_{e}, \mathbf{b}^{(\ell)}_{e}, \mathbf{W}^{(\ell)}_{d}, \mathbf{b}^{(\ell)}_{d}$ with pre-train autoencoder\;
    Initialize $\bm{\mu}$ with $K$-means on the representations learned by pre-train autoencoder\;
    Initialize $\mathbf{W}^{(\ell)}$ randomly\;
    \For{$iter \in 0,1,\cdots,MaxIter $}
    {
      Generate DNN representations $\mathbf{H}^{(1)}, \mathbf{H}^{(2)}, \cdots, \mathbf{H}^{(L)}$\;
      Use $\mathbf{H}^{(L)}$ to compute distribution $Q$ via Eq. \ref{qij}\;
      Calculate target distribution $P$ via Eq. \ref{pij}\;
      \For{$\ell \in 1,\cdots,L $}
      {
        Use the delivery operator with $\epsilon$=0.5 $\mathbf{\widetilde{Z}}^{(\ell)}=\frac{1}{2} \mathbf{Z}^{(\ell)} + \frac{1}{2}\mathbf{H}^{(\ell)}$\;
        Generate the next GCN layer representation
        $\mathbf{Z}^{(\ell+1)}=\phi\left(\mathbf{\widetilde{D}}^{-\frac{1}{2}} \mathbf{\widetilde{A}} \mathbf{\widetilde{D}}^{-\frac{1}{2}} \mathbf{\widetilde{Z}}^{(\ell)} \mathbf{W}^{(\ell)}_{g} \right)$\;
      }
      Calculate the distribution $Z$ via Eq. \ref{zij}\;
      Feed $\mathbf{H}^{(L)}$ to the decoder to construct the raw data $\mathbf{X}$\;
      Calculate $\mathcal{L}_{res}$, $\mathcal{L}_{clu}$, $\mathcal{L}_{gcn}$, respectively\;
      Calculate the loss function via Eq. \ref{loss}\;
      Back propagation and update parameters in SDCN\;
    }
    Calculate the clustering results based on distribution $Z$\;
    return $\mathbf{R}$\;
\end{algorithm}

Through this mechanism, SDCN can directly concentrate two different objectives, i.e. clustering objective and classification objective, in one loss function. And thus, the overall loss function of the our proposed SDCN is:
\begin{align}
\mathcal{L}=\mathcal{L}_{res} + \alpha \mathcal{L}_{clu} + \beta \mathcal{L}_{gcn},
\label{loss}
\end{align}
where $\alpha>0$ is a hyper-parameter that balances the clustering optimization and local structure preservation of raw data and $\beta>0$ is a coefficient that controls the disturbance of GCN module to the embedding space. 

In practice, after training until the maximum epochs, SDCN will get a stable result. Then we can set labels to samples. We choose the soft assignments in distribution $Z$ as the final clustering results. Because the representations learned by GCN contains two different kinds of information. The label assigned to sample $i$ is:
\begin{align}
r_{i}=\mathop{\arg\max}_{j} \ \ z_{ij},
\label{maxq}
\end{align}
where $z_{ij}$ is calculated in Eq. \ref{zij}.

The algorithm of the whole model is shown in Algorithm 1.

\subsection{Theory Analysis}
In this section, we will analyze how SDCN introduces structural information into the autoencoder. Before that, we give the definition of graph regularization and second-order graph regularization.

\begin{definition}
Graph regularization \cite{belkin2003laplacian}. Given a weighted graph $\mathcal{G}$, the objective of graph regularization is to minimize the following equation:
\begin{align}
\sum_{ij}\frac{1}{2}\left\|\mathbf{h}_{i}-\mathbf{h}_{j}\right\|_{2}^{2}w_{ij},
\end{align}
where $w_{ij}$ means the weight of the edge between node \textit{i} and node \textit{j}, and $\mathbf{h}_{i}$ is the representation of node \textit{i}.
\label{gr}
\end{definition}

Based on Definition \ref{gr}, we can find that the graph regularization indicates that if there is a larger weight between nodes $i$ and $j$, their representations should be more similar.

\begin{definition}
Second-order similarity. We assume that $\mathbf{A}$ is the adjacency matrix of graph $\mathcal{G}$ and $\mathbf{a}_{i}$ is the $i$-th column of $\mathbf{A}$. The second-order similarity between node \textit{i} and node \textit{j} is
\begin{align}
s_{ij}=\frac{\mathbf{a}_{i}^{T}\mathbf{a}_{j}}{\left\|\mathbf{a}_{i}\right\| \left\|\mathbf{a}_{j}\right\|}=\frac{\mathbf{a}_{i}^{T}\mathbf{a}_{j}}{\sqrt{d_{i}}\sqrt{d_{j}}}=\frac{\mathcal{{C}}}{\sqrt{d_{i}}\sqrt{d_{j}}},
\end{align}
where $\mathcal{C}$ is the number of common neighbors between node \textit{i} and node \textit{j} and $d_{i}$ is the degree of node \textit{i}.
\end{definition}

\begin{definition}
Second-order graph regularization. The objective of second-order graph regularization is to minimize the equation
\begin{align}
\sum_{i,j}\frac{1}{2}\left\|\mathbf{h}_{i}-\mathbf{h}_{j}\right\|_{2}^{2}s_{ij},
\end{align}
where $s_{ij}$ is the second-order similarity.
\label{sgr}
\end{definition}

Compared with Definition \ref{gr}, Definition \ref{sgr} imposes a high-order constraint, i.e., if two nodes have many common neighbors, their representations should also be more similar.

\begin{theorem}
GCN provides an approximate second-order graph regularization for the DNN representations.
\end{theorem}

\begin{proof}
Here we focus on the $\ell$-th layer of SDCN. $\mathbf{h}_{i}$ is the $i$-th row of $\mathbf{H}^{(\ell)}$, representing the data representation of sample $i$ learned by autoencoder and $\hat{\mathbf{h}}_{i}=\phi \left( \sum_{j \in \mathcal{N}_{i}}\frac{\mathbf{h}_{j}}{\sqrt{d_{i}}\sqrt{d_{j}}} \mathbf{W}\right)$ is the representation $\mathbf{h}_{i}$ passing through the GCN layer. Here we assume that $\phi(x)=x$ and $\mathbf{W}=\mathbf{I}$, and $\hat{\mathbf{h}}_{i}$ can be seen as the average of neighbor representations. Hence we can divide $\hat{\mathbf{h}}_{i}$ into three parts: the node representations $\frac{\mathbf{h}_{i}}{d_{i}}$, the sum of common neighbor representations $\mathcal{S}=\sum_{p\in\mathcal{N}_{i}\cap\mathcal{N}_{j}} \frac{\mathbf{h}_{p}}{\sqrt{d_{p}}}$ and the sum of non-common neighbor representations $\mathcal{D}_{i}=\sum_{q\in\mathcal{N}_{i}-\mathcal{N}_{i}\cap\mathcal{N}_{j}} \frac{\mathbf{h}_{q}}{\sqrt{d_{q}}}$, where $\mathcal{N}_{i}$ is the neighbors of node \textit{i}. The distance between the representations $\hat{\mathbf{h}}_{i}$ and $\hat{\mathbf{h}}_{j}$ is:
\begin{equation}
\label{regularization}
\begin{aligned}
& \left\| \hat{\mathbf{h}}_{i} - \hat{\mathbf{h}}_{j} \right\|_{2}\\
& = \left\| \left( \frac{\mathbf{h}_{i}}{d_{i}} - \frac{\mathbf{h}_{j}}{d_{j}} \right) + \left( \frac{\mathcal{S}}{\sqrt{d_{i}}} - \frac{\mathcal{S}}{\sqrt{d_{j}}} +  \right) + \left( \frac{\mathcal{D}_{i}}{\sqrt{d_{i}}} - \frac{\mathcal{D}_{j}}{\sqrt{d_{j}}} \right) \right\|_{2} \\
& \leq \left\| \frac{\mathbf{h}_{i}}{d_{i}} - \frac{\mathbf{h}_{j}}{d_{j}} \right\|_{2} + \left| \frac{\sqrt{d_{i}} - \sqrt{d_{j}}}{\sqrt{d_{i}}\sqrt{d_{j}}} \right| \left\| \mathcal{S} \right\|_{2} + \left\| \frac{\mathcal{D}_{i}}{\sqrt{d_{i}}} - \frac{\mathcal{D}_{j}}{\sqrt{d_{j}}} \right\|_{2} \\
& \leq \left\| \frac{\mathbf{h}_{i}}{d_{i}} - \frac{\mathbf{h}_{j}}{d_{j}} \right\|_{2} + \left| \frac{\sqrt{d_{i}} - \sqrt{d_{j}}}{\sqrt{d_{i}}\sqrt{d_{j}}} \right| \left\| \mathcal{S} \right\|_{2} + \left( \left\| \frac{\mathcal{D}_{i}}{\sqrt{d_{i}}} \right\|_{2} + \left\| \frac{\mathcal{D}_{j}}{\sqrt{d_{j}}} \right\|_{2} \right).\\
\end{aligned}
\end{equation}
We can find that the first term of Eq. \ref{regularization} is independent of the second-order similarity. Hence the upper bound of the distance between two node representations is only related to the second and third terms. For the second item of Eq. \ref{regularization}, if $d_{i} \ll d_{j}$, $w_{ij} \le \sqrt \frac{d_{i}}{d_{j}}$, which is very small and not consistent with the precondition. If $d_{i} \approx d_{j}$, the effect of the second item is paltry and can be ignored. For the third item of Eq. \ref{regularization}, if two nodes have many common neighbors, their non-common neighbors will be very few, and the values of $\left\| \frac{\mathcal{D}_{i}}{\sqrt{d_{i}}} \right\|_{2}$ and $\left\| \frac{\mathcal{D}_{j}}{\sqrt{d_{j}}} \right\|_{2}$ are positively correlated with non-common neighbors. If the second-order similarity $s_{ij}$ is large, the upper bound of $\left\| \hat{\mathbf{h}}_{i} - \hat{\mathbf{h}}_{j} \right\|_{2}$ will drop. In an extreme case, i.e. $w_{ij}=1$, $\left\| \hat{\mathbf{h}}_{i} - \hat{\mathbf{h}}_{j} \right\|_{2}=\frac{1}{d}\left\| \mathbf{h}_{i} - \mathbf{h}_{j} \right\|_{2}.$
\end{proof}

This shows that after the DNN representations pass through the GCN layer, if the nodes with large second-order similarity, GCN will force the representations of nodes to be close to each other, which is same to the idea of second-order graph regularization.

\begin{theorem}
The representation $Z^{(\ell)}$ learned by SDCN is equivalent to the sum of the representations with different order structural information.
\end{theorem}

\begin{proof} For the simplicity of the proof, let us assume that $\phi\left(x\right)=x$, $\mathbf{b}_{e}^{(\ell)}=\mathbf{0}$ and $\mathbf{W}_{g}^{(\ell)}=\mathbf{I}, \ \forall \ell \in \left[ 1,2,\cdots,L\right]$. We can rewrite Eq. \ref{ll} as
\begin{equation}
\begin{aligned}
& \mathbf{Z}^{(\ell+1)}= \mathbf{\widehat{A}}\mathbf{\widetilde{Z}}^{(\ell)}\\
& \mathbf{Z}^{(\ell+1)}= (1-\epsilon)\mathbf{\widehat{A}}\mathbf{Z}^{(\ell)}+\epsilon\mathbf{\widehat{A}}\mathbf{H}^{(\ell)},
\end{aligned}
\end{equation}
where $\mathbf{\widehat{A}}=\mathbf{\widetilde{D}}^{-\frac{1}{2}} \mathbf{\widetilde{A}} \mathbf{\widetilde{D}}^{-\frac{1}{2}}$. After $L$-th propagation step, the result is
\begin{equation}
\begin{aligned}
\mathbf{Z}^{(L)}= \left(1-\epsilon\right)^{L}\mathbf{\widehat{A}}^{L}\mathbf{X}+\epsilon\sum_{\ell=1}^{L}\left(1-\epsilon\right)^{\ell-1}\mathbf{\widehat{A}}^{\ell}\mathbf{H}^{(\ell)}.
\end{aligned}
\label{over-smoothing}
\end{equation}
Note that $\mathbf{\widehat{A}}^{L}\mathbf{X}$ is the output of standard GCN, which may suffer from the over-smoothing problem. Moreover, if $L\to\infty$, the left term tends to 0 and the right term dominants the data representation. We can clearly see that the right term is the sum of different representations, i.e. $\mathbf{H}^{(\ell)}$, with different order structural information.
\end{proof}

The advantages of the delivery operator in Eq. \ref{add} are two-folds: one is that the data representation $Z^{(\ell)}$ learn by SDCN contains different structural information. Another is that it can alleviate the over-smoothing phenomenon in GCN. Because multilayer GCNs focus on higher-order information, the GCN module in SDCN is the sum of the representations with different order structural information. Similar to \cite{klicpera2018predict}, our method also uses the fusion of different order information to alleviate the over-smoothing phenomenon in GCN. However, different from \cite{klicpera2018predict} treating different order adjacency matrices with the same representations, our SDCN gives different representations to different order adjacency matrices. This makes our model incorporate more information.

\subsection{Complexity Analysis}
In this work, we denote $d$ as the dimension of the input data and the dimension of each layer of the autoencoder is $d_{1}, d_{2}, \cdots, d_{L}$. The size of weight matrix in the first layer of the encoder is $\mathbf{W}^{(1)}_{e} \in \mathbb{R}^{d \times d_{1}}$. $N$ is the number of the input data. The time complexity of the autoencoder is $\mathcal{O}(Nd^{2}d^{2}_{1}...d^{2}_{L})$. As for the GCN module, because the operation of GCN can be efficiently implemented using sparse matrix, the time complexity is linear with the number of edges $|\mathcal{E}|$. The time complexity is $\mathcal{O}(|\mathcal{E}|dd_{1}...d_{L})$. In addition, we suppose that there are $K$ classes in the clustering task, so the time complexity of the Eq. \ref{qij} is $\mathcal{O}(NK+N\log N)$ corresponding to \cite{xie2016unsupervised}. The overall time complexity of our model is $\mathcal{O}(Nd^{2}d^{2}_{1}...d^{2}_{L}+|\mathcal{E}|dd_{1}...d_{L}+NK+N\log N)$, which is linearly related to the number of samples and edges.

\section{EXPERIMENTS}

\begin{table}
\centering
\caption{The statistics of the datasets.}
\begin{tabular}{@{}c|c|c|c|c@{}}
\toprule
\textbf{Dataset} & \textbf{Type} & \textbf{Samples} & \textbf{Classes} & \textbf{Dimension} \\ \midrule
\textbf{USPS}    & Image    & 9298             & 10               & 256                \\
\textbf{HHAR}    & Record    & 10299            & 6                & 561                \\
\textbf{Reuters}    & Text    & 10000            & 4                & 2000                \\
\textbf{ACM}     & Graph         & 3025             & 3                & 1870               \\
\textbf{DBLP}     & Graph         & 4058             & 4                & 334               \\
\textbf{Citeseer}     & Graph         & 3327             & 6                & 3703               \\\bottomrule
\end{tabular}
\label{statistics}
\end{table}

\begin{table*}[tp]
\centering
\caption{Clustering results on six datasets (mean$\pm$std). The bold numbers represent the best results and the numbers with asterisk are the best results of the baselines.}
\label{tab:graph}
\resizebox{\textwidth}{47mm}
{
\begin{tabular}{c|c|ccccccc|cc}
\hline
Dataset & Metric & $K$-means & AE & DEC & IDEC & GAE & VGAE & DAEGC & SDCN$_{Q}$ & SDCN \\
\hline
\multirow{4}{*}{USPS} & ACC & 66.82$\pm$0.04 & 71.04$\pm$0.03 & 73.31$\pm$0.17 & 76.22$\pm$0.12$^{*}$ & 63.10$\pm$0.33 & 56.19$\pm$0.72 & 73.55$\pm$0.40 & 77.09$\pm$0.21 & \textbf{78.08$\pm$0.19} \\
 & NMI & 62.63$\pm$0.05 & 67.53$\pm$0.03 & 70.58$\pm$0.25 & 75.56$\pm$0.06$^{*}$ & 60.69$\pm$0.58 & 51.08$\pm$0.37 & 71.12$\pm$0.24 & 77.71$\pm$0.21 & \textbf{79.51$\pm$0.27} \\
 & ARI & 54.55$\pm$0.06 & 58.83$\pm$0.05 & 63.70$\pm$0.27 & 67.86$\pm$0.12$^{*}$ & 50.30$\pm$0.55 & 40.96$\pm$0.59 & 63.33$\pm$0.34 & 70.18$\pm$0.22 & \textbf{71.84$\pm$0.24} \\
 & F1 & 64.78$\pm$0.03 & 69.74$\pm$0.03 & 71.82$\pm$0.21 & 74.63$\pm$0.10$^{*}$ & 61.84$\pm$0.43 & 53.63$\pm$1.05 & 72.45$\pm$0.49 & 75.88$\pm$0.17 & \textbf{76.98$\pm$0.18} \\ \hline
\multirow{4}{*}{HHAR} & ACC & 59.98$\pm$0.02 & 68.69$\pm$0.31 & 69.39$\pm$0.25 & 71.05$\pm$0.36 & 62.33$\pm$1.01 & 71.30$\pm$0.36 & 76.51$\pm$2.19$^{*}$ & 83.46$\pm$0.23 & \textbf{84.26$\pm$0.17} \\
 & NMI & 58.86$\pm$0.01 & 71.42$\pm$0.97 & 72.91$\pm$0.39 & 74.19$\pm$0.39$^{*}$ & 55.06$\pm$1.39 & 62.95$\pm$0.36 & 69.10$\pm$2.28 & 78.82$\pm$0.28 & \textbf{79.90$\pm$0.09} \\
 & ARI & 46.09$\pm$0.02 & 60.36$\pm$0.88 & 61.25$\pm$0.51 & 62.83$\pm$0.45$^{*}$ & 42.63$\pm$1.63 & 51.47$\pm$0.73 & 60.38$\pm$2.15 & 71.75$\pm$0.23 & \textbf{72.84$\pm$0.09} \\
 & F1 & 58.33$\pm$0.03 & 66.36$\pm$0.34 & 67.29$\pm$0.29 & 68.63$\pm$0.33 & 62.64$\pm$0.97 & 71.55$\pm$0.29 & 76.89$\pm$2.18$^{*}$ & 81.45$\pm$0.14 & \textbf{82.58$\pm$0.08} \\ \hline
\multirow{4}{*}{Reuters} & ACC & 54.04$\pm$0.01 & 74.90$\pm$0.21 & 73.58$\pm$0.13 & 75.43$\pm$0.14$^{*}$ & 54.40$\pm$0.27 & 60.85$\pm$0.23 & 65.50$\pm$0.13 & \textbf{79.30$\pm$0.11} & 77.15$\pm$0.21 \\
 & NMI & 41.54$\pm$0.51 & 49.69$\pm$0.29 & 47.50$\pm$0.34 & 50.28$\pm$0.17$^{*}$ & 25.92$\pm$0.41 & 25.51$\pm$0.22 & 30.55$\pm$0.29 & \textbf{56.89$\pm$0.27} & 50.82$\pm$0.21 \\
 & ARI & 27.95$\pm$0.38 & 49.55$\pm$0.37 & 48.44$\pm$0.14 & 51.26$\pm$0.21$^{*}$ & 19.61$\pm$0.22 & 26.18$\pm$0.36 & 31.12$\pm$0.18 & \textbf{59.58$\pm$0.32} & 55.36$\pm$0.37 \\
 & F1 & 41.28$\pm$2.43 & 60.96$\pm$0.22 & 64.25$\pm$0.22$^{*}$ & 63.21$\pm$0.12 & 43.53$\pm$0.42 & 57.14$\pm$0.17 & 61.82$\pm$0.13 & \textbf{66.15$\pm$0.15} & 65.48$\pm$0.08 \\ \hline
\multirow{4}{*}{ACM} & ACC & 67.31$\pm$0.71 & 81.83$\pm$0.08 & 84.33$\pm$0.76 & 85.12$\pm$0.52 & 84.52$\pm$1.44 & 84.13$\pm$0.22 & 86.94$\pm$2.83$^{*}$ & 86.95$\pm$0.08 & \textbf{90.45$\pm$0.18} \\
 & NMI & 32.44$\pm$0.46 & 49.30$\pm$0.16 & 54.54$\pm$1.51 & 56.61$\pm$1.16$^{*}$ & 55.38$\pm$1.92 & 53.20$\pm$0.52 & 56.18$\pm$4.15 & 58.90$\pm$0.17 & \textbf{68.31$\pm$0.25} \\
 & ARI & 30.60$\pm$0.69 & 54.64$\pm$0.16 & 60.64$\pm$1.87 & 62.16$\pm$1.50$^{*}$ & 59.46$\pm$3.10 & 57.72$\pm$0.67 & 59.35$\pm$3.89 & 65.25$\pm$0.19 & \textbf{73.91$\pm$0.40} \\
 & F1 & 67.57$\pm$0.74 & 82.01$\pm$0.08 & 84.51$\pm$0.74 & 85.11$\pm$0.48 & 84.65$\pm$1.33 & 84.17$\pm$0.23 & 87.07$\pm$2.79$^{*}$ & 86.84$\pm$0.09 & \textbf{90.42$\pm$0.19} \\ \hline
\multirow{4}{*}{DBLP} & ACC & 38.65$\pm$0.65 & 51.43$\pm$0.35 & 58.16$\pm$0.56 & 60.31$\pm$0.62 & 61.21$\pm$1.22 & 58.59$\pm$0.06 & 62.05$\pm$0.48$^{*}$ & 65.74$\pm$1.34 & \textbf{68.05$\pm$1.81} \\
 & NMI & 11.45$\pm$0.38 & 25.40$\pm$0.16 & 29.51$\pm$0.28 & 31.17$\pm$0.50 & 30.80$\pm$0.91 & 26.92$\pm$0.06 & 32.49$\pm$0.45$^{*}$ & 35.11$\pm$1.05 & \textbf{39.50$\pm$1.34} \\
 & ARI & 6.97$\pm$0.39 & 12.21$\pm$0.43 & 23.92$\pm$0.39 & 25.37$\pm$0.60$^{*}$ & 22.02$\pm$1.40 & 17.92$\pm$0.07 & 21.03$\pm$0.52 & 34.00$\pm$1.76 & \textbf{39.15$\pm$2.01} \\
 & F1 & 31.92$\pm$0.27 & 52.53$\pm$0.36 & 59.38$\pm$0.51 & 61.33$\pm$0.56 & 61.41$\pm$2.23 & 58.69$\pm$0.07 & 61.75$\pm$0.67$^{*}$ & 65.78$\pm$1.22 & \textbf{67.71$\pm$1.51} \\ \hline
\multirow{4}{*}{Citeseer} & ACC & 39.32$\pm$3.17 & 57.08$\pm$0.13 & 55.89$\pm$0.20 & 60.49$\pm$1.42 & 61.35$\pm$0.80 & 60.97$\pm$0.36 & 64.54$\pm$1.39$^{*}$ & 61.67$\pm$1.05 & \textbf{65.96$\pm$0.31} \\
 & NMI & 16.94$\pm$3.22 & 27.64$\pm$0.08 & 28.34$\pm$0.30 & 27.17$\pm$2.40 & 34.63$\pm$0.65 & 32.69$\pm$0.27 & 36.41$\pm$0.86$^{*}$ & 34.39$\pm$1.22 & \textbf{38.71$\pm$0.32} \\
 & ARI & 13.43$\pm$3.02 & 29.31$\pm$0.14 & 28.12$\pm$0.36 & 25.70$\pm$2.65 & 33.55$\pm$1.18 & 33.13$\pm$0.53 & 37.78$\pm$1.24$^{*}$ & 35.50$\pm$1.49 & \textbf{40.17$\pm$0.43} \\
 & F1 & 36.08$\pm$3.53 & 53.80$\pm$0.11 & 52.62$\pm$0.17 & 61.62$\pm$1.39 & 57.36$\pm$0.82 & 57.70$\pm$0.49 & 62.20$\pm$1.32$^{*}$ & 57.82$\pm$0.98 & \textbf{63.62$\pm$0.24} \\ \hline
\end{tabular}
}
\label{results}
\end{table*}

\subsection{Datasets}
Our proposed SDCN is evaluated on six datasets. The statistics of these datasets are shown in Table \ref{statistics} and the detailed descriptions are the followings:

\begin{itemize}%[leftmargin=*]
\item \textbf{USPS}\cite{le1990handwritten}: The USPS dataset contains 9298 gray-scale handwritten digit images with size of 16x16 pixels. The features are the gray value of pixel points in images and all features are normalized to [0, 2].

\item \textbf{HHAR}\cite{stisen2015smart}: The Heterogeneity Human Activity Recognition (HHAR) dataset contains 10299 sensor records from smart phones and smart watches. All samples are partitioned into 6 categories of human activities, including biking, sitting, standing, walking, stair up and stair down.

\item \textbf{Reuters}\cite{lewis2004rcv1}: It is a text dataset containing around 810000 English news stories labeled with a category tree. We use 4 root categories: corporate/industrial, government/social, markets and economics as labels and sample a random subset of 10000 examples for clustering.

\item \textbf{ACM\footnotemark[2]}\cite{HAN}: This is a paper network from the ACM dataset. There is an edge between two papers if they are written by same author. Paper features are the bag-of-words of the keywords. We select papers published in KDD, SIGMOD, SIGCOMM, MobiCOMM and divide the papers into three classes (database, wireless communication, data mining) by their research area.

\footnotetext[2]{http://dl.acm.org/}

\item \textbf{DBLP\footnotemark[3]}\cite{HAN}: This is an author network from the DBLP dataset. There is an edge between two authors if they are the co-author relationship. The authors are divided into four areas: database, data mining, machine learning and information retrieval. We label each author's research area according to the conferences they submitted. Author features are the elements of a bag-of-words represented of keywords.

\footnotetext[3]{https://dblp.uni-trier.de}

\item \textbf{Citeseer\footnotemark[4]}: It is a citation network which contains sparse bag-of-words feature vectors for each document and a list of citation links between documents. The labels contain six area: agents, artificial intelligence, database, information retrieval, machine language, and HCI.

\footnotetext[4]{http://citeseerx.ist.psu.edu/index}

\end{itemize}

\subsection{Baselines}
We compare our proposed method SDCN with three types of methods, including clustering methods on raw data, DNN-based clustering methods and GCN-based graph clustering methods.

\begin{itemize}%[leftmargin=*]
\item \textbf{$K$-means} \cite{hartigan1979algorithm}: A classical clustering method based on the raw data.

\item \textbf{AE} \cite{hinton2006reducing}: It is a two-stage deep clustering algorithm which performs $K$-means on the representations learned by autoencoder.

\item \textbf{DEC} \cite{xie2016unsupervised}: It is a deep clustering method which designs a clustering objective to guide the learning of the data representations.

\item \textbf{IDEC} \cite{guo2017improved}: This method adds a reconstruction loss
to DEC, so as to learn better representation.

\item \textbf{GAE \& VGAE} \cite{kipf2016variational}: It is an unsupervised graph embedding method using GCN to learn data representations.

\item \textbf{DAEGC} \cite{wang2019attributed}: It uses an attention network to learn the node representations and employs a clustering loss to supervise the process of graph clustering.

\item \textbf{SDCN$_{Q}$}: The variant of SDCN with distribution $Q$.

\item \textbf{SDCN}: The proposed method.
\end{itemize}

\textbf{Metrics.} We employ four popular metrics: Accuracy (ACC), Normalized Mutual Information (NMI), Average Rand Index (ARI) and macro F1-score (F1). For each metric, a larger value implies a better clustering result.

\textbf{Parameter Setting.} We use the pre-trained autoencoder for all DNN-based clustering methods (AE+$K$-means, DEC, IDEC) and SDCN. We train the autoencoder end-to-end using all data points with 30 epochs and the learning rate is $10^{-3}$. In order to be consistent with previous methods \cite{xie2016unsupervised, guo2017improved}, we set the dimension of the autoencoder to $d$-500-500-2000-10, where $d$ is the dimension of the input data. The dimension of the layers in GCN module is the same to the autoencoder. As for the GCN-based methods, we set the dimension of GAE and VAGE to $d$-256-16 and train them with 30 epochs for all datasets. For DAEGC, we use the setting of \cite{wang2019attributed}. In hyperparameter search, we try $\left\{ 1, 3, 5 \right\}$ for the update interval in DEC and IDEC, $\left\{1, 0.1, 0.01, 0.001\right\}$ for the hyperparameter $\gamma$ in IDEC and report the best results. For our SDCN, we uniformly set $\alpha=0.1$ and $\beta=0.01$ for all the datasets because our method is not sensitive to hyperparameters. For the non-graph data, we train the SDCN with 200 epochs, and for graph data, we train it with 50 epochs. Because the graph structure with prior knowledge, i.e. citation network, contains more information than KNN graph, which can accelerate convergence speed. The batch size is set to 256 and learning rate is set to $10^{-3}$ for USPS, HHAR, ACM, DBLP and $10^{-4}$ for Reuters, Citeseer. For all methods using $K$-means algorithm to generate clustering assignments, we initialize 20 times and select the best solution. We run all methods 10 times and report the average results to prevent extreme cases.

\subsection{Analysis of Clustering Results}
Table \ref{results} shows the clustering results on six datasets. Note that in USPS, HHAR and Reuters, we use the KNN graph as the input of the GCN module, while for ACM, DBLP and Citeseer, we use the original graph. We have the following observations:

\begin{itemize}%[leftmargin=*]
\item For each metric, our methods SDCN and SDCN$_{Q}$ achieve the best results in all the six datasets. In particular, compared with the best results of the baselines, our approach achieves a significant improvement of 6\% on ACC, 17\% on NMI and 28\% on ARI averagely. The reason is that SDCN successfully integrates the structural information into deep clustering and the dual self-supervised module guides the update of autoencoder and GCN, making them enhance each other.

\item SDCN generally achieves better cluster results than SDCN$_{Q}$. The reason is that SDCN uses the representations containing the structural information learned by GCN, while SDCN$_{Q}$ mainly uses the representations learned by the autoencoder. However, in Reuters, the result of SDCN$_{Q}$ is much better than SDCN. Because in the KNN graph of Reuters, many different classes of nodes are connected together, which contains much wrong structural information. Therefore, an important prerequisite for the application of GCN is to construct a KNN graph with less noise.

\item Clustering results of autoencoder based methods (AE, DEC, IDEC) are generally better than those of GCN-based methods (GAE, VAGE, DAEGC) on the data with KNN graph, while GCN-based methods usually perform better on the data with graph structure. The reason is that GCN-based methods only use structural information to learn the data representation. When the structural information in the graph is not clear enough, e.g. KNN graph, the performance of the GCN-based methods will decline. Besides, SDCN integrates structural information into deep clustering, so its clustering performance is better than these two methods.

\item Comparing the results of AE with DEC and the results of GAE with DAEGC, we can find that the clustering loss function, defined in Eq. \ref{clu}, plays an important role in improving the deep clustering performance. Because IDEC and DAEGC can be seen as the combination of the clustering loss with AE and GAE, respectively. It improves the cluster cohesion by making the data representation closer to the cluster centers, thus improving the clustering results.
\end{itemize}

\subsection{Analysis of Variants}
We compare our model with two variants to verify the ability of GCN in learning structural information and the effectiveness of the delivery operator. Specifically, we define the following variants:

\begin{figure}[htbp]
\centering
\subfigure[Datasets with KNN graph]{
\includegraphics[width=4cm]{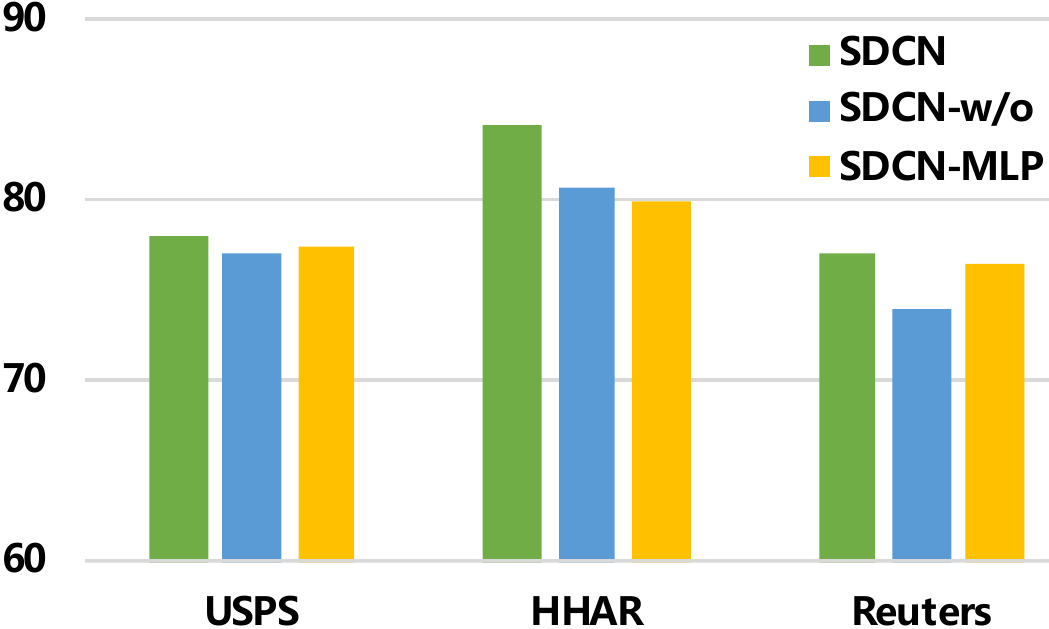}
\label{vknn}
}
\subfigure[Datasets with original graph]{
\includegraphics[width=4cm]{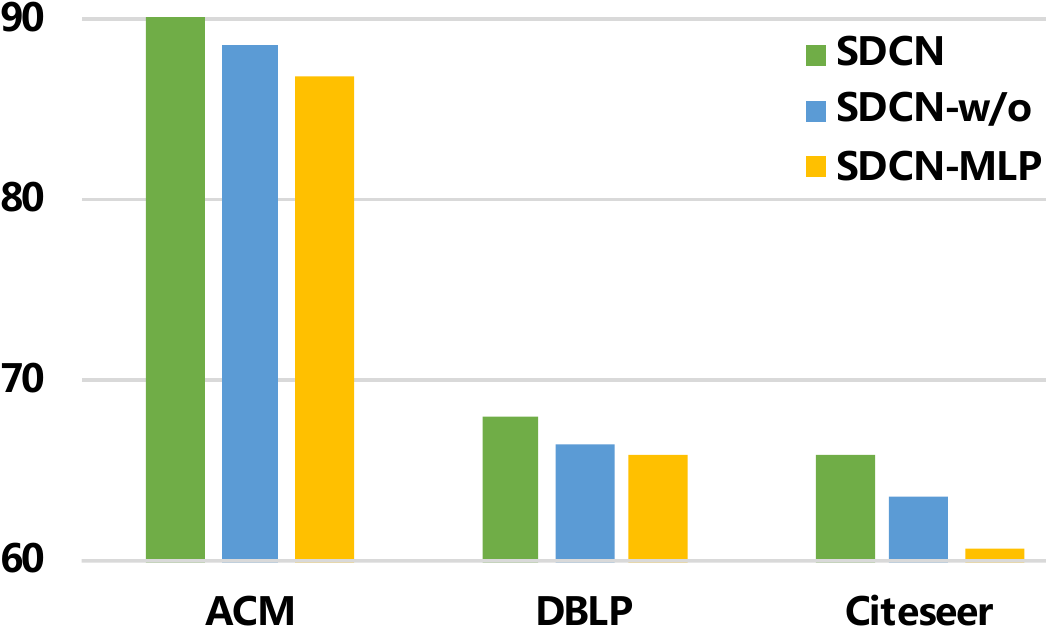}
\label{vatt}
}
\caption{Clustering accuracy with different variants}
\label{variants}
\end{figure}

\begin{itemize}%[leftmargin=*]
\item \textbf{SDCN-w/o:} This variant is SDCN without delivery operator, which is used to validate the effectiveness of our proposed delivery operator.

\item \textbf{SDCN-MLP:} This variant is SDCN replacing the GCN module with the same number of layers of multilayer perceptron (MLP), which is used to validate the advantages of GCN in learning structural information.
\end{itemize}

From Figure \ref{variants}, we have the following observations:

\begin{itemize}%[leftmargin=*]
\item In Figure \ref{vknn}, we can find that the clustering accuracy of SDCN-MLP is better than SDCN-w/o in Reuters and achieves similar results in USPS and HHAR. This shows that in the KNN graph, without delivery operator, the ability of GCN in learning structural information is severely limited. The reason is that multilayer GCN will produce a serious over-smoothing problem, leading to the decrease of the clustering results. On the other hand, SDCN is better than SDCN-MLP. This proves that the delivery operator can help GCN alleviate the over-smoothing problem and learn better data representation.

\item In Figure \ref{vatt}, we can find that the clustering accuracy of SDCN-w/o is better than SDCN-MLP in all three datasets containing original graph. This shows that GCN has the powerful ability in learning data representation with structural information. Besides, SDCN performs better than SDCN-w/o in the three datasets. This proves that there still exists over-smoothing problem in the SDCN-w/o, but the good graph structure still makes SDCN-w/o achieve not bad clustering results.

\item Comparing the results in Figure \ref{vknn} and Figure \ref{vatt}, we can find no matter on which types of datasets, SDCN achieves the best performance, compared with SDCN-w/o and SDCN-MLP. This proves that both the delivery operator and GCN play an important role in improving clustering quality.
\end{itemize}

\begin{table}
\caption{Effect of different propagation layers ($L$)}
\begin{tabular}{@{}c|c|cccc@{}}
\toprule
\multicolumn{2}{c|}{} & ACC & NMI & ARI & F1 \\ \midrule
\multicolumn{1}{l|}{\multirow{4}{*}{ACM}} & SDCN-4 & \textbf{90.45} & \textbf{68.31} & \textbf{73.91} & \textbf{90.42} \\
\multicolumn{1}{l|}{} & SDCN-3 & 89.06 & 64.86 & 70.51 & 89.03 \\
\multicolumn{1}{l|}{} & SDCN-2 & 89.12 & 66.48 & 70.94 & 89.04 \\
\multicolumn{1}{l|}{} & SDCN-1 & 77.69 & 51.59 & 50.13 & 74.62 \\ \midrule
\multirow{4}{*}{DBLP} & SDCN-4 & \textbf{68.05} & \textbf{39.51} & \textbf{39.15} & \textbf{67.71} \\
 & SDCN-3 & 65.11 & 36.81 & 36.03 & 64.98 \\
 & SDCN-2 & 66.72 & 37.19 & 37.58 & 65.37 \\
 & SDCN-1 & 64.19 & 30.69 & 33.62 & 60.44 \\ \midrule
\multirow{4}{*}{Citeseer} & SDCN-4 & \textbf{65.96} & \textbf{38.71} & \textbf{40.17} & \textbf{61.62} \\
 & SDCN-3 & 59.18 & 32.11 & 32.16 & 55.92 \\
 & SDCN-2 & 60.96 & 33.69 & 34.49 & 57.31 \\
 & SDCN-1 & 58.58 & 32.91 & 32.31 & 52.38 \\ \bottomrule
\end{tabular}
\label{layers}
\end{table}

\subsection{Analysis of Different Propagation Layers}
To investigate whether SDCN benefits from multilayer GCN, we vary the depth of the GCN module while keeping the DNN module unchanged. In particular, we search the number of layers in the range of $\left\{1,2,3,4\right\}$. 
There are a total of four layers in the encoder part of the DNN module in SDCN, generating the representation $\mathbf{H}^{(1)}$, $\mathbf{H}^{(2)}$, $\mathbf{H}^{(3)}$, $\mathbf{H}^{(4)}$, respectively. SDCN-$L$ represents that there is a total of $L$ layers in the GCN module. 
For example, SDCN-2 means that $\mathbf{H}^{(3)}$, $\mathbf{H}^{(4)}$ will be transferred to the corresponding GCN layers for propagating.
We choose the datasets with original graph to verify the effect of the number of the propagation layers on the clustering effect because they have the nature structural information. From Table \ref{layers}, we have the following observations:

\begin{itemize}%[leftmargin=*]
\item Increasing the depth of SDCN substantially enhances the clustering performance. It is clear that SDCN-2, SDCN-3 and SDCN-4 achieve consistent improvement over SDCN-1 in all the across. Besides, SDCN-4 performs better than other methods in all three datasets. Because the representations learned by each layer in the autoencoder are different, to preserve information as much as possible, we need to put all the representations learned from autoencoder into corresponding GCN layers.

\item There is an interesting phenomenon that the performance of SDCN-3 is not as good as SDCN-2 in all the datasets. The reason is that SDCN-3 uses the representation $\mathbf{H}^{(2)}$, which is a middle layer of the encoder. The representation generated by this layer is in the transitional stage from raw data to semantic representation, which inevitably loses some underlying information and lacks of semantic information. Another reason is that GCN with two layers does not cause serious over-smoothing problems, proved in \cite{li2018deeper}. For SDCN-3, due to the number of layers is not enough, the over-smoothing term in Eq. \ref{over-smoothing} is not small enough so that it is still plagued by the over-smoothing problems.
\end{itemize}

\begin{figure}[htbp]
\centering
\subfigure[USPS]{
\begin{minipage}[t]{0.3\linewidth}
\includegraphics[width=2.7cm]{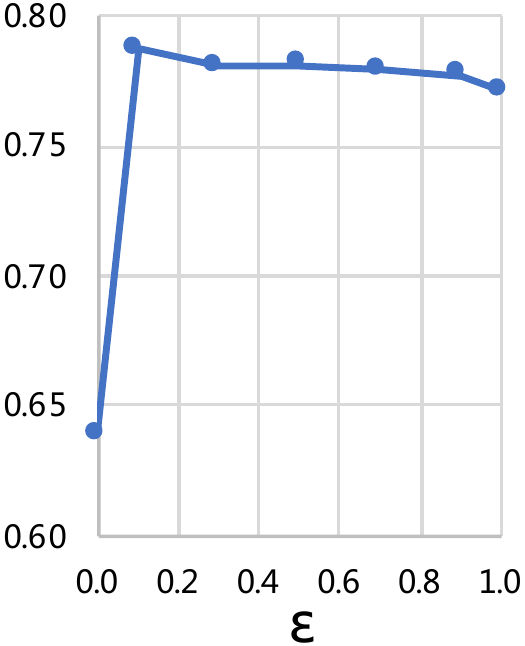}
\end{minipage}
}
\subfigure[HHAR]{
\begin{minipage}[t]{0.3\linewidth}
\includegraphics[width=2.7cm]{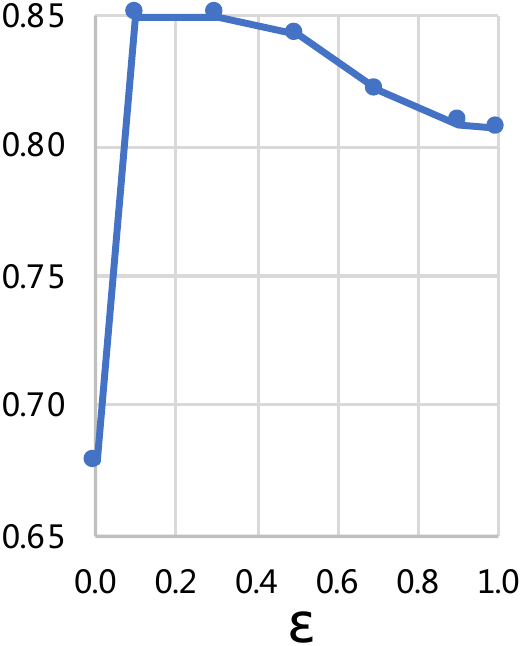}
\end{minipage}
}
\subfigure[Reuters]{
\begin{minipage}[t]{0.3\linewidth}
\includegraphics[width=2.7cm]{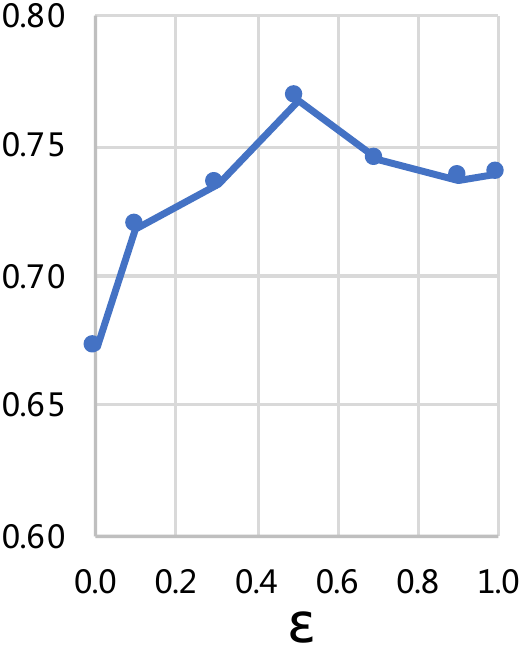}
\end{minipage}
}
\subfigure[ACM]{
\begin{minipage}[t]{0.3\linewidth}
\includegraphics[width=2.7cm]{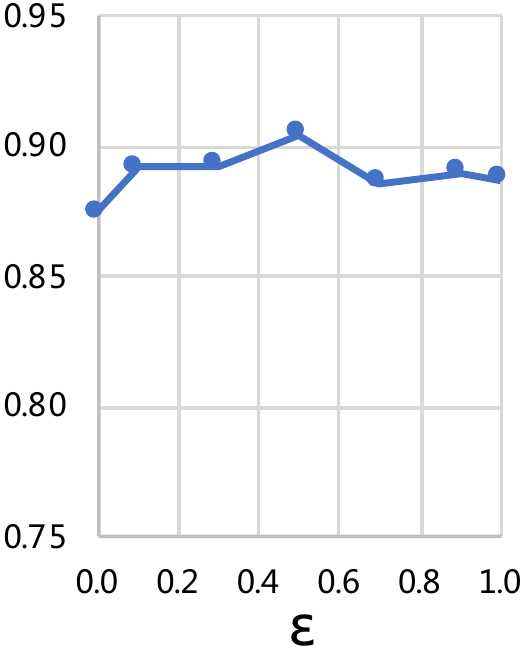}
\end{minipage}
}
\subfigure[DBLP]{
\begin{minipage}[t]{0.3\linewidth}
\includegraphics[width=2.7cm]{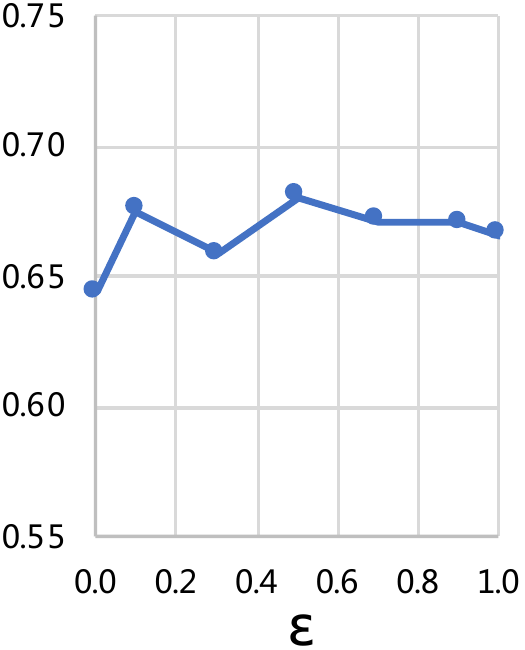}
\end{minipage}
}
\subfigure[Citeseer]{
\begin{minipage}[t]{0.3\linewidth}
\includegraphics[width=2.7cm]{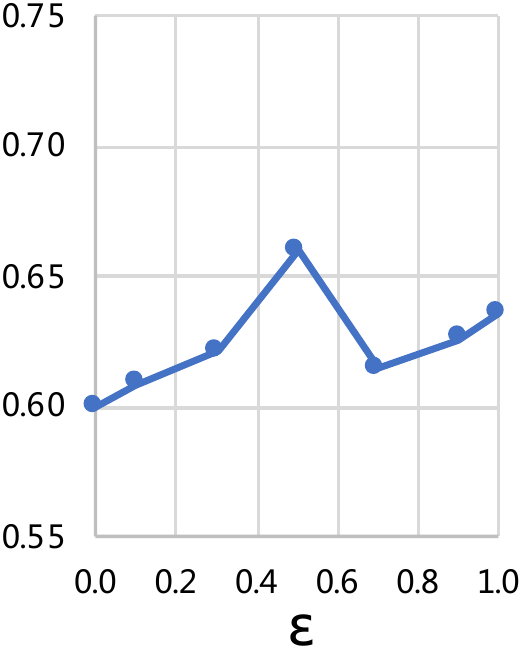}
\end{minipage}
}
\caption{Clustering accuracy with different $\epsilon$}
\end{figure}

\subsection{Analysis of balance coefficient $\epsilon$}
In previous experiments, in order to reduce hyperparameter search, we uniformly set the balance coefficient $\epsilon$ to 0.5. In this experiment, we will explore how SDCN is affected by different $\epsilon$ on different datasets. In detail, we set $\epsilon=\left\{0.0, 0.1, 0.3, 0.5, 0.7, 0.9, 1.0\right\}$. Note that $\epsilon=0.0$ means the representations in GCN module do not contain the representation from autoencoder and $\epsilon=1.0$ represents that GCN only use the representation $\mathbf{H}^{(L)}$ learned by DNN. From Figure 4, we can find:

\begin{itemize}%[leftmargin=*]
\item Clustering accuracy with parameter $\epsilon=0.5$ in four datasets (Reuters, ACM, DBLP, Citeseer) achieve the best performance, which shows that the representations of GCN module and DNN module are equally important and the improvement of SDCN depends on the mutual enhancement of the two modules.

\item Clustering accuracy with parameter $\epsilon=0.0$ in all datasets performs the worst. Clearly, when $\epsilon=0.0$, the GCN module is equivalent to the standard multilayer GCN, which will produce very serious over-smoothing problem \cite{li2018deeper}, leading to the decline of the clustering quality. 
Compared with the accuracy when $\epsilon=0.1$, we can find that even injecting a small amount of representations learned by autoencoder into GCN can help alleviate the over-smoothing problem.

\item Another interesting observation is that SDCN with parameter $\epsilon=1.0$ still gets a higher clustering accuracy. The reason is that although SDCN with parameter $\epsilon=1.0$ only use the representation $\mathbf{H}^{(L)}$, it contains the most important information of the raw data. After passing through one GCN layer, it can still achieve some structural information to improve clustering performance. However, due to the limitation of the number of layers, the results are not the best.
\end{itemize}

\begin{figure*}[ht]
\centering
\subfigure[USPS]{
\includegraphics[width=4cm]{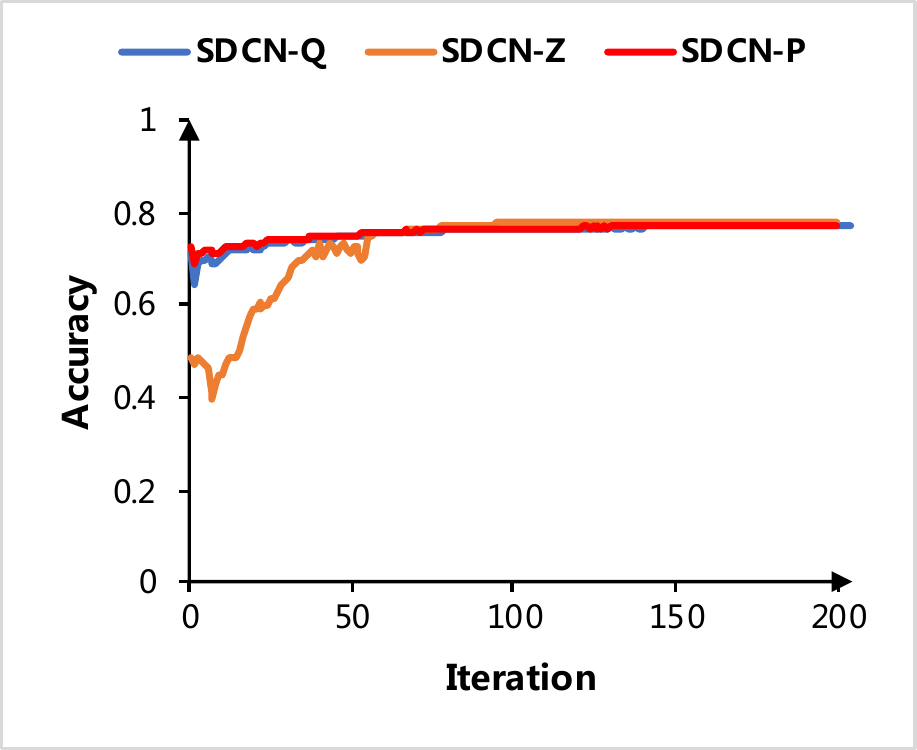}
}
\subfigure[HHAR]{
\includegraphics[width=4cm]{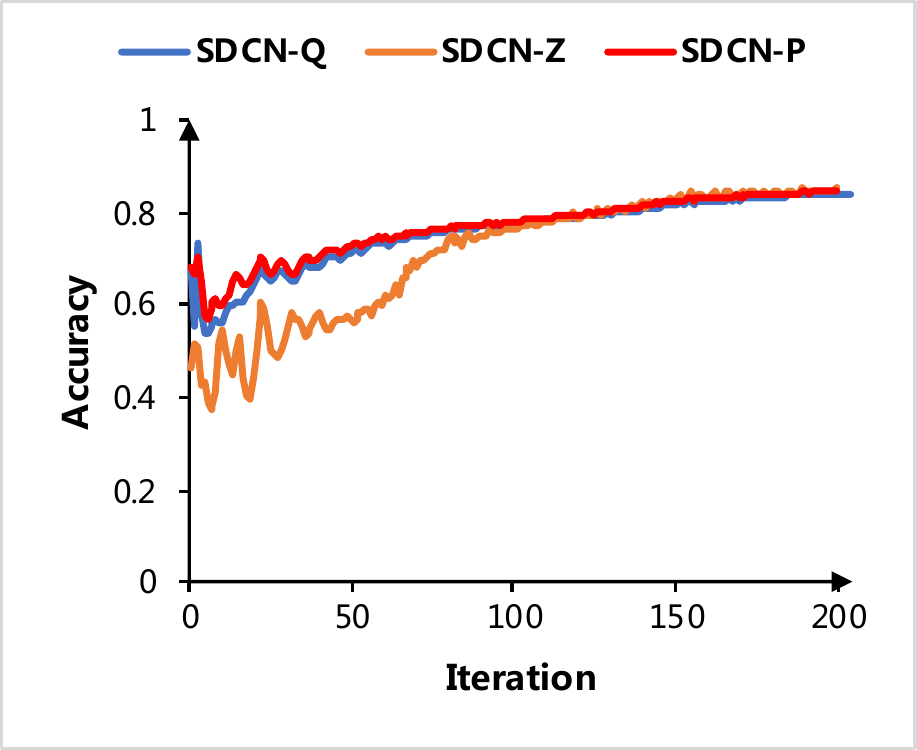}
}
\subfigure[ACM]{
\includegraphics[width=4cm]{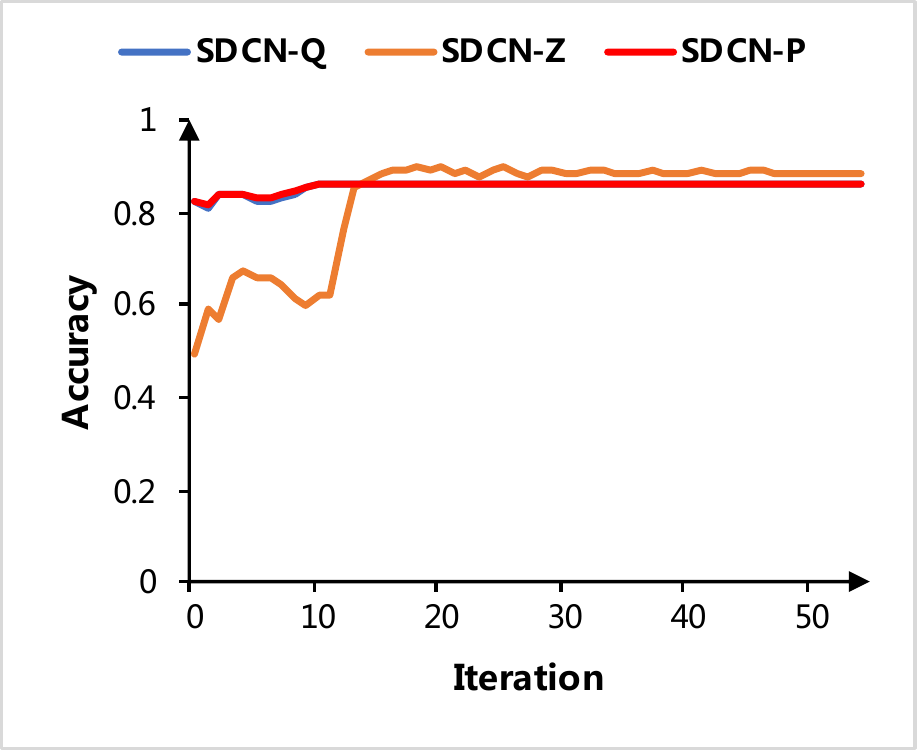}
}
\subfigure[DBLP]{
\includegraphics[width=4cm]{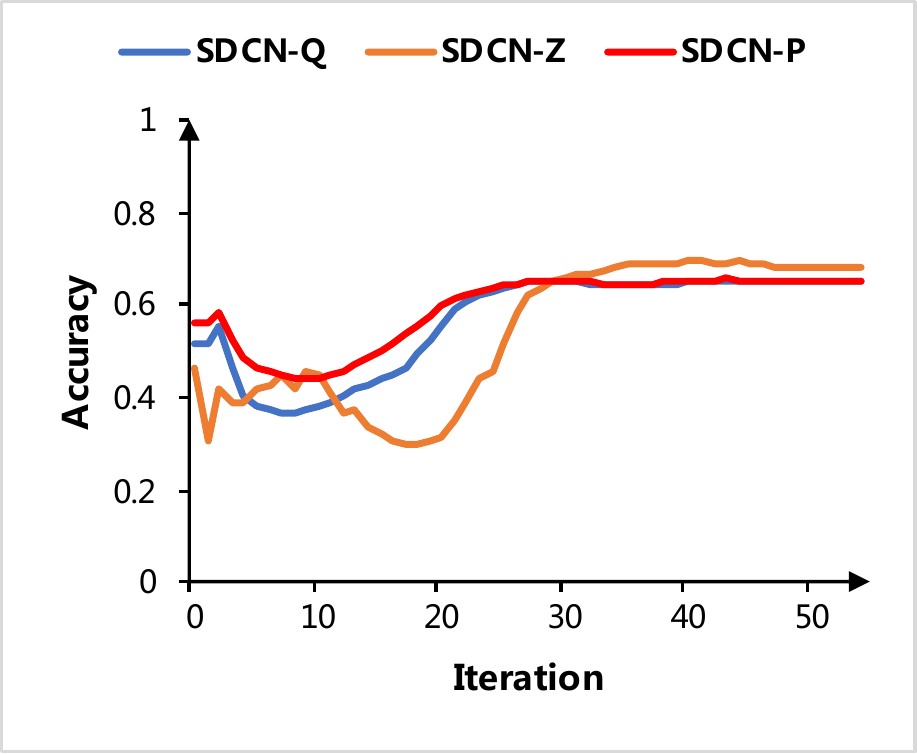}
}
\caption{Training process on different datasets}
\label{trainprocess}
\end{figure*}

\begin{figure}
\centering
\subfigure[Accuracy on USPS]{
\begin{minipage}[t]{0.3\linewidth}
\centering
\includegraphics[width=2.7cm]{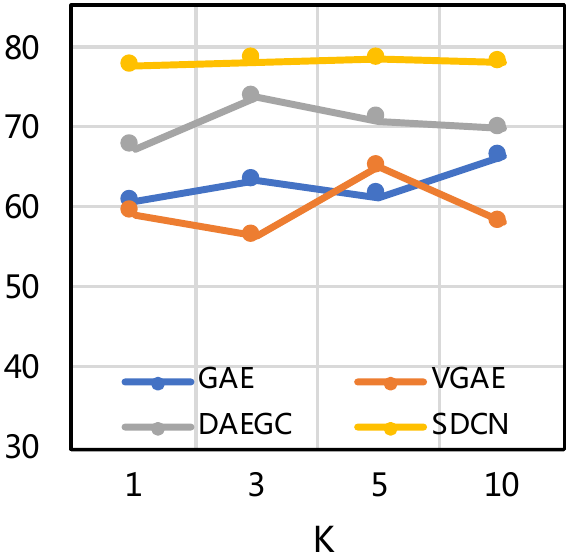}
\end{minipage}
}
\subfigure[Accuracy on HHAR]{
\begin{minipage}[t]{0.3\linewidth}
\centering
\includegraphics[width=2.7cm]{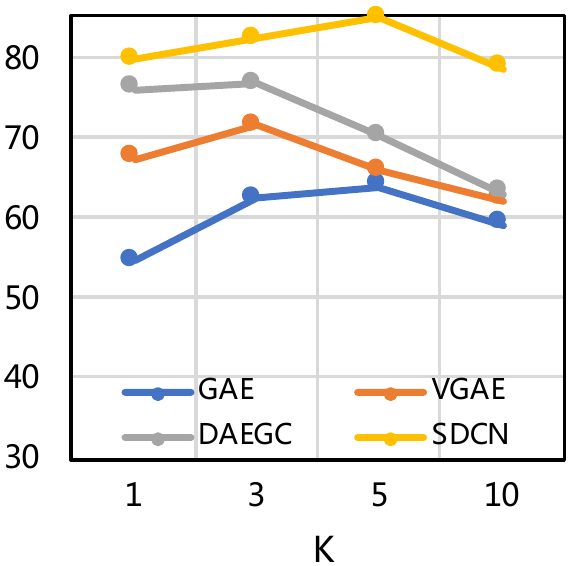}
\end{minipage}
}
\subfigure[Accuracy on Reuters]{
\begin{minipage}[t]{0.3\linewidth}
\centering
\includegraphics[width=2.7cm]{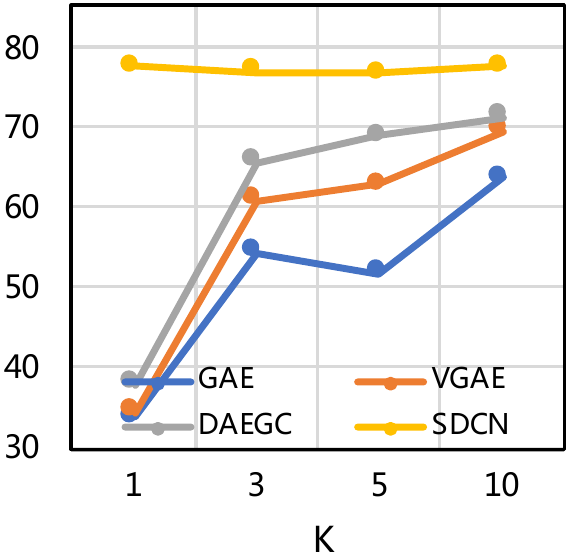}
\end{minipage}
}
\quad
\subfigure[NMI on USPS]{
\begin{minipage}[t]{0.3\linewidth}
\centering
\includegraphics[width=2.7cm]{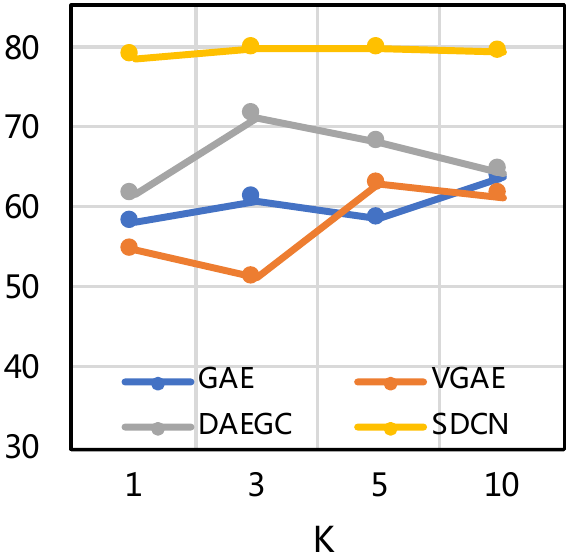}
\end{minipage}
}
\subfigure[NMI on HHAR]{
\begin{minipage}[t]{0.3\linewidth}
\centering
\includegraphics[width=2.7cm]{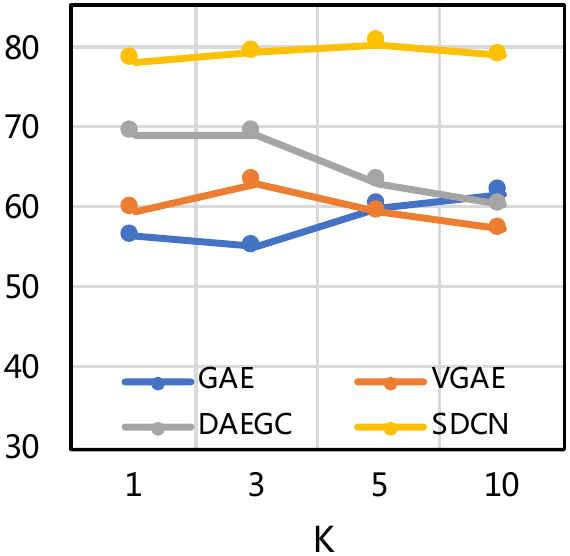}
\end{minipage}
}
\subfigure[NMI on Reuters]{
\begin{minipage}[t]{0.3\linewidth}
\centering
\includegraphics[width=2.7cm]{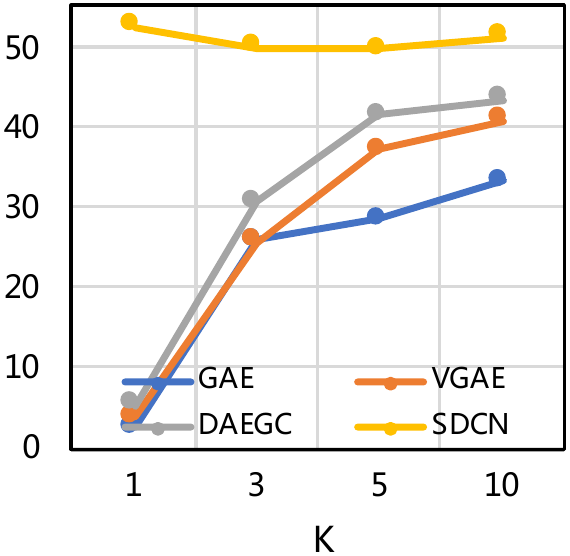}
\end{minipage}
}
\caption{Clustering results with different K}
\label{ksen}
\end{figure}

\subsection{$K$-sensitivity Analysis}
Since the number of the nearest neighbors $K$ is an important parameter in the construction of the KNN graph, we design a $K$-sensitivity experiment on the datasets with KNN graph. This experiment is mainly to prove that our model is $K$-insensitive. Hence we compare SDCN with the clustering methods focusing on the graph data (GAE, VGAE, DAEGC). From Figure \ref{ksen}, we can find that with $K$=$\left\{1,3,5,10\right\}$, our proposed SDCN is much better than GAE, VGAE and DAEGC, which proves that our method can learn useful structural information even in the graphs containing noise. Another finding is that these four methods can achieve good performance when K = 3 or K = 5, but in the case of K = 1 and K = 10, the performance will drop significantly. The reason is that when K = 1, the KNN graph contains less structural information and when K = 10, the communities in KNN graph are over-lapping. In summary, SDCN can achieve stable results compared with other baseline methods on the KNN graphs with different number of nearest neighbors.

\subsection{Analysis of Training Process}
In this section, we analyze the training progress in different datasets. Specifically, we want to explore how the cluster accuracy of the three sample assignments distributions in SDCN varies with the number of iterations. In Figure \ref{trainprocess}, the red line SDCN-P, the blue line SDCN-Q and the orange line SDCN-Z represent the accuracy of the target distribution $P$, distribution $Q$ and distribution $Z$, respectively. 
In most cases, the accuracy of SDCN-P is higher than that of SDCN-Q, which shows that the target distribution $P$ is able to guide the update of the whole model. At the beginning, the results of the accuracy of three distributions all decrease in different ranges. Because the information learned by autoencoder and GCN is different, it may rise a conflict between the results of the two modules, making the clustering results decline. Then the accuracy of SDCN-Q and SDCN-Z quickly increase to a high level, because the target distribution SDCN-P eases the conflict between the two modules, making their results tend to be consistent. In addition, we can see that with the increase of training epochs, the clustering results of SDCN tend to be stable and there is no significant fluctuation, indicating the good robustness of our proposed model.

\section{CONCLUSION}
In this paper, we make the first attempt to integrate the structural information into deep clustering. We propose a novel structural deep clustering network, consisting of DNN module, GCN module, and dual self-supervised module. Our model is able to effectively combine the autoencoder-spectific representation with GCN-spectific representation by a delivery operator. Theoretical analysis is provided to demonstrate the strength of the delivery operator. We show that our proposed model consistently outperforms the state-of-the-art deep clustering methods in various open datasets.

%% The acknowledgments section is defined using the "acks" environment
%% (and NOT an unnumbered section). This ensures the proper
%% identification of the section in the article metadata, and the
%% consistent spelling of the heading.
\begin{acks}
This work is supported by the National Key Research and Development Program of China (2018YFB1402600) and the National Natural Science Foundation of China (No. 61772082, 61702296, 61806020, 61972442, U1936104). It is also supported by 2018 Tencent Marketing Solution Rhino-Bird Focused Research Program.
\end{acks}

%%
%% The next two lines define the bibliography style to be used, and
%% the bibliography file.
\clearpage
\bibliographystyle{ACM-Reference-Format}
\bibliography{ref}

\end{document}